\documentclass[a4paper,thmsb,11pt]{article}
\usepackage{color}
\usepackage{geometry}
\usepackage{geometry}
\usepackage{amsmath}
\usepackage{amssymb}
\usepackage{times}
\usepackage{mathrsfs}
\usepackage{graphicx}
\usepackage{animate}
\usepackage{algorithm}
\usepackage{algorithmicx}
\usepackage{algpseudocode}
\usepackage{tikz}
\usetikzlibrary{chains, positioning}
\usepackage[colorlinks,linkcolor=blue]{hyperref}
\newtheorem{theorem}{Theorem}

\newtheorem{problem}{Problem}
\newtheorem{definition}{Definition}
\newtheorem{corollary}{Corollary}
\newtheorem{lemma}{Lemma}

\begin{document}
%
\title{\textbf{Reachable Set Computation and Safety Verification for	Neural Networks with ReLU Activations}}

\author{Weiming~Xiang \footnotemark[1],~~Hoang-Dung Tran \footnotemark[1],~~and~~Taylor T. Johnson \footnotemark[1]}

\renewcommand{\thefootnote}{\fnsymbol{footnote}}
	
\footnotetext[1]{Authors are with the Department of Electrical Engineering and Computer Science, Vanderbilt University, Nashville, Tennessee 		37212, USA. Email: xiangwming@gmail.com (Weiming Xiang); Hoang-Dung Tran (trhoangdung@gmail.com); taylor.johnson@gmail.com (Taylor T. Johnson).}

\maketitle

\begin{abstract}
\boldmath
	Neural networks have been widely used to solve complex real-world problems. Due to the complicate, nonlinear, non-convex nature of neural networks, formal safety guarantees for the output behaviors of neural networks will be crucial for their applications in safety-critical systems. 
	In this paper,  the output reachable set computation and safety verification problems for a class of neural networks consisting of Rectified Linear Unit (ReLU) activation functions are addressed. A layer-by-layer approach is developed to compute output reachable set. The computation is formulated in the form of a set of manipulations for a union of polyhedra, which can be efficiently applied with the aid of polyhedron computation tools. Based on the output reachable set computation results, the safety verification for a ReLU neural network can be performed by checking the intersections of unsafe regions and output reachable set described by a union of polyhedra. A numerical example of a randomly generated ReLU neural network is provided to show the effectiveness of the approach developed in this paper.

\end{abstract}

\section{Introduction}
Artificial neural networks have been widely used in machine learning systems. Applications include adaptive control \cite{wang2016combined,wu2014exponential,he2016adaptive,sun2017adaptive,ge1999adaptive,hunt1992neural}, pattern recognition \cite{schmidhuber2015deep,lawrence1997face}, game playing \cite{silver2016mastering}, autonomous vehicles \cite{bojarski2016end} ,and many others. Though neural networks have been showing effectiveness and powerful ability in resolving complex problems, they are confined to systems which comply only to the lowest safety integrity levels since, in most of time, a neural network is viewed as a \emph{black box} without effective  methods to assure safety specifications for its outputs. Neural networks are trained over a finite number of input and output data, and are expected to be able to generalize to produce desirable outputs for given inputs even including previously unseen inputs. However, in many practical applications, the number of inputs is essentially infinite, this means it is impossible to check all the possible inputs only by performing experiments and moreover, it has been observed that neural networks can react in unexpected and incorrect ways to even slight perturbations of their inputs \cite{szegedy2013intriguing}, which could result in unsafe systems. Hence, methods that are able to provide formal guarantees are in a great demand for verifying specifications or properties of neural networks.  Verifying neural networks is a hard problem, even simple properties about them have been proven NP-complete
problems \cite{katz2017reluplex}. The difficulties mainly come from the presence of activation functions and the complex structures, making neural networks large-scale, nonlinear, non-convex and thus incomprehensible to humans. Until now, only few results have been reported for verifying neural networks. The verification for feed-forward
multi-layer neural networks is investigated based on \emph{Satisfiability Modulo Theory} (SMT) in \cite{huang2016safety,pulina2012challenging}. In \cite{pulina2010abstraction} an abstraction-refinement approach is proposed for verification of
specific networks known as \emph{Multi-Layer Perceptrons} (MLPs). In \cite{katz2017reluplex}, a specific kind of activation functions called \emph{Rectified Linear Unit} (ReLU) is considered for verification of neural networks. A simulation-based approach is developed in \cite{xiang2017output}, which turns the reachable set estimation problem into a neural network maximal sensitivity computation problem that is described in terms of a chain of convex optimization problems. Additionally, some recent reachable set estimation results are reported for neural networks \cite{xu2017reachable,zuo2014non,thuan2016reachable}, these results that are based on Lyapunov functions certainly have potentials to be further extended to safety verification.

A neural network is comprised of a set of layers of neurons, with a linear combination of values
from nodes in the preceding layer and applying an activation function to
the result. These activation functions are usually nonlinear functions. In this work, we are going to focus on ReLU activation functions  \cite{nair2010rectified}, which is widely used in many neural networks \cite{glorot2011deep,jarrett2009best,maas2013rectifier,huang2016safety}. A ReLU function is basically piecewise linear. It returns zero when the node is with a negative value, implying the node is inactive. On the other hand, when the node is active with a positive value, the ReLU returns the value unchanged.  This piecewise linearity allows ReLU neural networks have several advantages such as a faster training process and avoiding gradient vanishing problem. As to output reachable set computation and verification problems addressed in this paper, this piecewise linearity will also play a fundamental role in the computing procedure.

The main contribution of this work is to develop an approach for computing output reachable set  of ReLU neural networks and further applying to safety verification problems. Under the assumption that the initial set is described by a union of ployhedra, the output reachable set is computed layer-by-layer via a set of manipulations of ployhedra. For a ReLU function, there are three cases from the view of the output vectors:
\begin{itemize}
	\item \emph{Case 1}: All the elements in the input vector are positive, thus the output is exactly equivalent to the input;
	\item \emph{Case 2}:  All the elements are non-positive so that the ReLU function produces a zero vector according to the definition of ReLU function;
	\item \emph{Case 3}: The input vector has both positive and non-positive elements. This is a much more intricate case which will be proved later that the outputs belong to a union of polyhedra, that is essentially non-convex.
\end{itemize}

These above three caese are able to fully characterize the output behaviors of a ReLU function and form the basic idea of computing the output reachable set for neural networks comprised of ReLU neurons. With the above classifications and a complete reachability analysis for ReLU functions, the output reachable set of a ReLU layer can be obtained case by case and in the expression of a union of polyhedra.  Then,  the approach is generalized from a single layer to a neural network consisting of multiple layers for output reachable set computation of neural networks.  Finally, the safety verification can be performed by checking if there is non-empty intersection between the output reachable set and unsafe regions. Since the output reachable set computed in this work is an exact one with respect to an  input set, the verification results are sound for both safe and unsafe conclusions. The main benefit of our approach is that all the computation processes are formulated in terms of operations on polyhedra, which can be efficiently solved by existing tool for manipulations on polyhedra.   

The remainder of the paper is organized as follows. The preliminaries for ReLU neural networks and problem formulations are given in Section II. The output reachability analysis for ReLU functions is studied in Section III. The main results, reachable set computation and verification for ReLU neural networks, are presented in Section IV. A numerical example is given in Section V to illustrate our approach, and we conclude in Section VI.

\emph{Notations:}  $\mathbb{R}$ denotes the field of real numbers, $\mathbb{R}^n$  stands for the vector space of all $n$-tuples of real numbers, $\mathbb{R}^{n\times n}$  is the space of $n\times n$  matrices with real entries.  $\mathrm{diag}\{\cdots\}$ denotes a block-diagonal matrix. $\left\|  \mathbf{x}  \right\|_\infty$  stands for infinity norm for vector $\mathbf{x} \in \mathbb{R}^{n}$ defined as $\left\|  \mathbf{x}  \right\|_\infty = \max\nolimits_{i=1,\ldots,n}{\left|x_i\right|}$.  $\mathbf{A}^{\top}$ denotes the transpose of matrix $\mathbf{A}$. 

\section{Preliminaries and Problem Formulation}
This section presents the mathematical model of neural networks with ReLU activations considered in this paper and formulates the problems to be studied.
\subsection{Neural Networks with ReLU Activations}
A neural network consists of a number of interconnected neurons. Each neuron is a simple processing element that responds to the weighted inputs it received from other neurons. In this paper, we consider the most popular and general feedforward neural networks called the Multi-Layer Perceptron (MLP). Generally, an MLP consists of three typical classes of layers: An input layer, that
serves to pass the input vector to the network, hidden layers of computation neurons, and
an output layer composed of at least a computation neuron to produce the output vector.

The action of a neuron depends on its activation function, which is described as
\begin{align} \label{neuron}
y_i = f\left(\sum\nolimits_{j=1}^{n}\omega_{ij} x_j + \theta_i\right)
\end{align}
where $x_j$ is the $j$th input of the $i$th neuron, $\omega_{ij}$ is the weight from the $j$th input to the $i$th neuron, $\theta_i$ is called the bias of the $i$th neuron, $y_i$ is the output of the $i$th neuron, $f(\cdot)$ is the activation function. The activation function is generally a nonlinear function  describing the reaction of $i$th neuron with inputs $x_j(t)$, $j=1,\cdots,n$. Typical activation functions include rectified linear unit, logistic, tanh, exponential linear unit, linear functions, etc.

An MLP has multiple layers,  each layer $\ell$, $1 \le \ell \le L $, has $n^{[\ell]}$ neurons.  In particular, layer $\ell =0$ is used to denote the input layer and $n^{[0]}$ stands for the number of inputs in the rest of this paper and thus $n^{[L]}$ stands for the last layer, that is the output layer. For a neuron $i$, $1 \le i \le n^{[\ell]}$ in layer $\ell$, the corresponding input vector is denoted by $\mathbf{x}^{[\ell]} \in \mathbb{R}^{n^{[\ell-1]}}$ and the weight matrix is
\begin{equation*}
\mathbf{W}^{[\ell]} = \left[\boldsymbol{\omega}_{1}^{[\ell]},\ldots,\boldsymbol{\omega}_{n^{[\ell]}}^{[\ell]}\right]^{\top}
\end{equation*}
where $\boldsymbol{\omega}_{i}^{[\ell]} \in \mathbb{R}^{n^{[\ell-1]}}$ is the weight vector. The bias vector for layer $\ell$ is
\begin{equation*} \boldsymbol{\theta}^{[\ell]}=\left[\theta_1^{[\ell]},\ldots,\theta_{n^{[\ell]}}^{[\ell]}\right]^{\top}.
\end{equation*}

The output vector of layer $\ell$ can be expressed as
\begin{equation*}
\mathbf{y}^{[\ell]}=f_{\ell}(\mathbf{W}^{[\ell]}\mathbf{x}^{[\ell]}+\boldsymbol{\theta}^{[\ell]})
\end{equation*}
where $f_{\ell}(\cdot)$ is the activation function for layer $\ell$.

For an MLP, the output of $\ell-1$ layer is the input of $\ell$ layer, and the mapping from the input of input layer $\mathbf{x}^{[0]} \in \mathbb{R}^{n^{[0]}}$ to the output of output layer $\mathbf{y}^{[L]} \in \mathbb{R}^{n^{[L]}}$ stands for the input-output relation of the MLP, denoted by
\begin{equation}\label{NN}
\mathbf{y}^{[L]} = F (\mathbf{x}^{[0]})
\end{equation}
where $F(\cdot) \triangleq f_L  \circ f_{L - 1}  \circ  \cdots  \circ f_1(\cdot) $.

In this work, we aim at a class of activation functions called \emph{Rectified Linear Unit} (ReLU), which is expressed as below:
\begin{equation} \label{relu}
f(x) = \mathrm{ReLU}(x) =
\left\{ {\begin{array}{*{20}c}
	0 & {x \le 0}  \\
	x & {x > 0}  \\
	\end{array} } \right..
\end{equation}

Thus, the output of a neuron considered in (\ref{neuron}) can be rewritten as
\begin{equation}  \label{reluLayer}
y_i = \mathrm{ReLU}\left(\sum\nolimits_{j=1}^{n}\omega_{ij} x_j + \theta_i\right)
\end{equation}
and the corresponding output vector of layer $\ell$ becomes
\begin{equation} \label{ReLUNN}
\mathbf{y}^{[\ell]}=\mathrm{ReLU}(\mathbf{W}^{[\ell]}\mathbf{x}^{[\ell]}+\boldsymbol{\theta}^{[\ell]})
\end{equation}
in which the ReLU function is applied element-wise.

In the most of real applications, an MLP is usually viewed as a \emph{black box} to generate a desirable
output with respect to a given input. However, regarding property verifications such as the safety verification, it has been observed that even a well-trained neural network can react in unexpected and incorrect ways to even slight perturbations of their inputs, which could result in unsafe systems. Thus, the output reachable set estimation of an MLP, which is able to cover all possible values of outputs, is necessary for the safety verification of an MLP and draw a safe or unsafe conclusion for an MLP.

\subsection{Problem Formulation}

Given an input set $\mathcal{X}$, the reachable set of neural network (\ref{NN}) is stated by the following definition.
\begin{definition} \label{reach_set}
	Given a neural network in the form of (\ref{NN}) and input $\mathbf{x}^{[0]}$ belonging to a set $\mathcal{X}^{[0]}$, the output reachable set of (\ref{NN}) is defined by
	\begin{equation}\label{reachset}
	\mathcal{Y} \triangleq \left\{\mathbf{y}^{[L]}  \mid \mathbf{y}^{[L]} = F (\mathbf{x}^{[0]}),~\mathbf{x}^{[0]} \in \mathcal{X}^{[0]}\right\}.
	\end{equation}
\end{definition}

In our work, the input set is considered to be a union of $N$ polyhedra, that is expressed as $\mathcal{X}^{[0]} = \bigcup_{s =1}^{N_0}{\mathcal{X}_s^{[0]}}$, where $\mathcal{X}_{s}^{[0]}$, $s = 1,\ldots,N_0$, are described by
\begin{equation}
\mathcal{X}^{[0]}_s \triangleq\left\{ \mathbf{x}^{[0]} \mid  \mathbf{A}_s^{[0]}\mathbf{x}^{[0]} \le \mathbf{b}_s^{[0]},~\mathbf{x} \in \mathbb{R}^{n^{[0]}}\right\},~s =1,\ldots,N_0.  \label{inputset}
\end{equation}

With respect to input set (\ref{inputset}),  the reachable set computation problem for neural network (\ref{NN}) with ReLU activations is given as below.

\begin{problem} \label{problem1}
	Given an input set $\mathcal{X}^{[0]}$ and a neural network in the form of (\ref{NN})  with ReLU activations described by (\ref{ReLUNN}), how to compute the reachable set $\mathcal{Y}$ defined by (\ref{reachset})?
\end{problem}

Then, we will focus on the safety verification for neural networks. The safety specification for output is expressed by a set defined in the output space, describing the safety requirement. For example,  in accordance to input set, the safety region can be also considered as a union of polyhedra defined in output space as $\mathcal{S} = \bigcup_{m =1}^{M}{\mathcal{S}_m}$, where $\mathcal{S}_{m}$, $m = 1,\ldots,M$, are given by
\begin{equation}\label{safety_region}
\mathcal{S}_m\triangleq\left\{\mathbf{y}^{[L]}\mid \mathbf{C}_m \mathbf{y}^{[L]} \le \mathbf{d}_m,~\mathbf{y} \in \mathbb{R}^{n^{[L]}}\right\},~m = 1,\ldots,M.
\end{equation}

The safety region in the form of (\ref{safety_region}) formalizes the safety requirements for output $\mathbf{y}^{[L]}$.
If output $\mathbf{y}^{[L]}$ belongs to safety region $\mathcal{S}$, we say the neural network is safe, otherwise, it is called unsafe.

\begin{definition}
	Given a neural network in the form of (\ref{NN}) and safety region $\mathcal{S}$, the MLP is safe if and only if the following condition is satisfied:
	\begin{equation}\label{safety}
	\mathcal{Y} \cap \neg \mathcal{S} = \emptyset
	\end{equation}
	where $\neg$ is the symbol for logical negation and $\mathcal{Y}$ is the output reachable set of MLP defined by (\ref{reachset}).
\end{definition}

Therefore, the safety verification problem for MLP with ReLU activations can be stated as follows.
\begin{problem}\label{problem2}
	Given an input set $\mathcal{X}$ by (\ref{inputset}), a safety specification by (\ref{safety_region}), a neural network in the form of (\ref{NN}) with ReLU activation by (\ref{ReLUNN}), how to check if condition (\ref{safety}) is satisfied?
\end{problem}

The above two linked problems are the main concerns to be addressed in the rest of this paper. The crucial step is to find an efficient way to compute the output reachable set for a ReLU neural network with a given input set. In the next sections, the main results will be presented for the output reachable set computation  and safety verification for  ReLU neural networks.

\section{Ouput Reach Set Computation of ReLU Functions}
In this section, we consider the output reachable set of a single ReLU function $\mathbf{y} = \mathrm{ReLU}(\mathbf{x})$ with an input set $\mathcal{X} \subset \mathbb{R}^{n}$.  Before presenting the result, an indicator vector $\mathbf{p} = [p_0,\ldots,p_{n}]$, $p_i \in \{0,1\}$, is introduced for the following derivation. In the indicator vector $\mathbf{p}$,  the element $p_i$ is valuated as below:
\begin{equation*}
p_i  = \left\{ {\begin{array}{*{20}c}
	0 & {x_i \le 0}  \\
	1 & {x_i > 0}  \\	
	\end{array} } \right.,~i=1,\ldots,n.
\end{equation*}

Considering all the valuations of $p_i$ in  $\mathbf{p}$, there are $2^n$ possible valuations in total, which are indexed as
\begin{equation*}
\left[ {\begin{array}{*{20}c}
	{\mathbf{p}_0 }  \\
	{\mathbf{p}_1 }  \\
	\vdots   \\
	{\mathbf{p}_{2^n-1 } }  \\
	\end{array} } \right] = \left[ {\begin{array}{*{20}c}
	0 & 0 & 0 & 0  \\
	0 & 0 & 0 & 1  \\
	\vdots  &  \vdots  &  \vdots  &  \vdots   \\
	1 & 1 & 1 & 1  \\
	\end{array} } \right].
\end{equation*}

Furthermore, each indicator vector from $\mathbf{p}_0$ to $\mathbf{p}_{2^n-1}$ are diagonalized and denoted as 
\begin{equation*}
\left[ {\begin{array}{*{20}c}
	{\mathbf{P}_0 }  \\
	{\mathbf{P}_1 }  \\
	\vdots   \\
	{\mathbf{P}_{2^n-1 } }  \\
	\end{array} } \right] = \left[ {\begin{array}{*{20}c}
	\mathrm{diag}(\mathbf{p}_0)  \\
	\mathrm{diag}(\mathbf{p}_1)  \\
	\vdots    \\
	\mathrm{diag}(\mathbf{p}_{2^n-1})  \\
	\end{array} } \right].
\end{equation*}

Now, we are ready to compute the output reachable set of ReLU function $\mathbf{y} = \mathrm{ReLU}(\mathbf{x})$  with an input $\mathbf{x} \in \mathcal{X} \subset \mathbb{R}^{n}$. For the input set, we have three cases listed below:
\begin{itemize}
	\item \emph{Case 1}:  All the elements are positive by the input in $\mathcal{X}$, that implies
	\begin{equation*}\label{eq:case1_x}
	\mathbf{x} \in \mathcal{X}^{+} = \left\{\mathbf{x} \mid \mathbf{x}  > 0,~\mathbf{x} \in \mathcal{X}\right\}.
	\end{equation*}
	
	According to the definition of ReLU function, the output set should be
	\begin{equation}\label{eq:Y+}
	\mathcal{Y}^{+} = \left\{\mathbf{y} \mid \mathbf{y} = \mathbf{x},~\mathbf{x} \in \mathcal{X}^{+}\right\}.
	\end{equation}
	\item \emph{Case 2}:  All the elements in the outputs are non-positive, which means
	\begin{equation*}
	\mathbf{x} \in \mathcal{X}^{-} = \left\{\mathbf{x} \mid \mathbf{x} \le \mathbf{0},~\mathbf{x} \in \mathcal{X}\right\}.
	\end{equation*}
	
	By the definition of ReLU, it directly leads to \begin{equation}\label{eq:Y-}
	\mathcal{Y}^{-}=\left\{\mathbf{y} \mid \mathbf{y}=\mathbf{0},~\mathbf{x} \in \mathcal{X}^-\right\}.
	\end{equation}
	
	\item \emph{Case 3}:  The outputs have both positive and non-positive elements, which corresponds to indicator vectors $\mathbf{p}_h$, $h=1,\ldots,2^{n}-2$.  Note that, for each $\mathbf{p}_h$, $h=1,\ldots,2^{n}-2$, the element $p_i = 0$ indicates $y_i=\mathrm{ReLU}(x_i)=  0$ due to $x_i \le 0$. With respect to each $\mathbf{p}_h$, $h=1,\ldots,2^{n}-2$ and noting  $\mathbf{x} = [x_1,\ldots,x_n] $, we define the set
	\begin{equation*}
	\mathcal{X}_{h}^{-} = \left\{\mathbf{x} \mid x_i \le 0 ,~\mathbf{x} \in \mathcal{X}\right\},~h=1,\ldots,2^{n}-2
	\end{equation*}
	where $i \in \left\{i \mid p_i =0,~ [p_1,\ldots,p_n] = \mathbf{p}_h\right\}$. In a compact form, it can be expressed as
	\begin{equation*}
	\mathcal{X}_{h}^{-} = \left\{\mathbf{x} \mid (\mathbf{I}-\mathbf{P}_h)\mathbf{x} \le 0 ,~\mathbf{x} \in \mathcal{X}\right\},~h=1,\ldots,2^{n}-2.
	\end{equation*}
	
	Due to  ReLU functions, when $x_i < 0$, it will be set to $0$, thus the output for $\mathbf{x} \in \mathcal{X}_{h}^{-}$ should be
	\begin{equation*}
	\mathbf{y}=\mathbf{P}_h\mathbf{x},~ \mathbf{x} \in \mathcal{X}_{h}^{-},~h=1,\ldots,2^{n}-2.
	\end{equation*}
	
	Again, due to ReLU functions, the final value should be non-negative, that is $\mathbf{y}\ge 0$, thus this additional  constraint has to be added for $\mathbf{x}$ to have
	\begin{equation*}
	\mathbf{x} \in \mathcal{X}_{h}^{+}=\left\{\mathbf{x} \mid \mathbf{P}_h\mathbf{x} \ge 0,~\mathbf{x} \in \mathcal{X}^-_{h}\right\},~h=1,\ldots,2^{n}-2.
	\end{equation*}
	
	As a result, the output reachable set is
	\begin{equation}\label{eq:Yh}
	\mathcal{Y}_{h} = \left\{\mathbf{y} \mid \mathbf{y} = \mathbf{x},~\mathbf{x} \in \mathcal{X}_{h}^{+} \right\},~h=1,\ldots,2^{n}-2.
	\end{equation}
\end{itemize}

An illustration for the above three cases is shown in Figure \ref{fig:ReLU} with two dimensional input space. In Figure \ref{fig:ReLU}, (a) is for the \emph{Case 1}, (b) is for the \emph{Case 2}, (c) and (d) are for the \emph{Case 3}.
Summarizing the three cases for a ReLU function $\mathbf{y} = \mathrm{ReLU}(\mathbf{x})$, the following proposition can be obtained.

\begin{figure}
	\begin{center}
		\includegraphics[width=11cm]{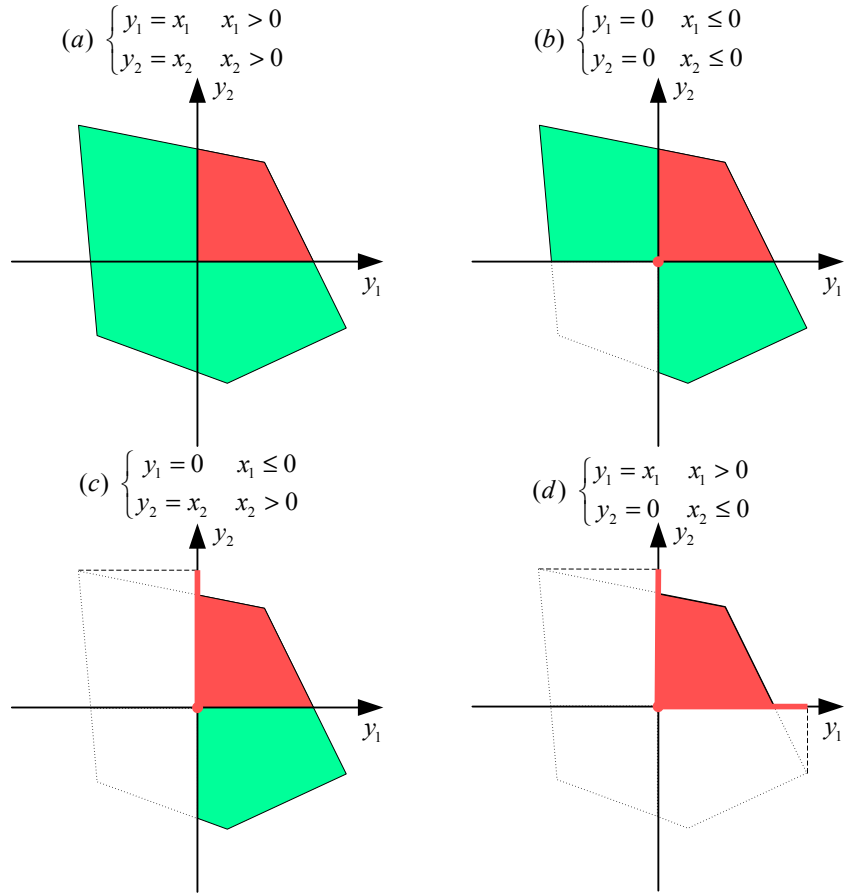}
		\caption{Visualization of ReLU function $\mathbf{y} = \mathrm{ReLU}(\mathbf{x})$, $x \in \mathcal{X} \triangleq \{\mathbf{x}\mid \mathbf{Ax} \le \mathbf{b},\mathbf{x} \in \mathbb{R}^2\}$. (a) is for the \emph{Case 1}, the red area is for $\mathcal{Y}^{+}$; (b) is for the \emph{Case 2} where the red spot denotes $\mathbf{y}= \mathbf{0}$; (c) and (d) are for the \emph{Case 3}, the red lines on the axis are $\mathcal{Y}_h$, $h=1,2$. The resulting output set in (d) is non-convex but expressed by a union of polyhedra. }
		\label{fig:ReLU}
	\end{center}
\end{figure}

\begin{theorem}\label{thm1}
	Given a ReLU function $\mathbf{y} = \mathrm{ReLU}(\mathbf{x})$ with an input set $\mathcal{X} \in \mathbb{R}^{n}$, its output reachable set $\mathcal{Y}$ is
	\begin{equation}
	\mathcal{Y} = \mathcal{Y}^{+} \cup \mathcal{Y}^{-} \cup \left(\bigcup\nolimits_{h=1}^{2^n-2}\mathcal{Y}_{h}\right)
	\end{equation}
	where $\mathcal{Y}^{+}$, $\mathcal{Y}^{-}$ and $\mathcal{Y}_{h}$ are defined by (\ref{eq:Y+}), (\ref{eq:Y-}) and (\ref{eq:Yh}), respectively.
\end{theorem}
\begin{proof}
	The proof can be straightforwardly obtained from the derivation of above three cases. Three cases completely characterize the behaviors of ReLU functions due to $\mathcal{X} = \mathcal{X}^{+} \cup \mathcal{X}^{-} \cup \left(\bigcup\nolimits_{h=1}^{2^n-2}\mathcal{X}_{h}^{-}\right)$. For the case of $\mathbf{x} \in \mathcal{X}^+$, it produces $\mathcal{Y^+}$, and $\mathbf{x} \in \mathcal{X}^-$ leads to $\mathcal{Y^-}$. As to $x \in \mathcal{X}_h^{-}$, $h=1,\ldots,2^n-2$, the output reachable set is $\bigcup\nolimits_{h=1}^{2^n-2}\mathcal{Y}_{h}$.   Thus, the output reachable set $\mathcal{Y}$ is the union of output sets of three cases, that is $	\mathcal{Y} = \mathcal{Y}^{+} \cup \mathcal{Y}^{-} \cup \left(\bigcup\nolimits_{h=1}^{2^n-2}\mathcal{Y}_{h}\right)$. 
\end{proof}

Theorem \ref{thm1} gives out a general result for output reachable set of a ReLU function, since there is no restriction imposed on input set $\mathcal{X}$. In the following, we consider the input set as a union of polyhedra described by $\mathcal{X}=\bigcup\nolimits_{s=1}^{N}\mathcal{X}_s$, where $\mathcal{X}_s$, $s \in \{1,\ldots,N\}$ are given as
\begin{equation}\label{inputset1}
\mathcal{X}_s \triangleq \left\{\mathbf{x}\mid \mathbf{A}_s\mathbf{x} \le \mathbf{b}_s,~\mathbf{x} \in \mathbb{R}^n\right\},~s =1,\ldots,N.
\end{equation}

Based on Theorem \ref{thm1}, the following result can be derived  for input sets described by a union of polyhedra.

\begin{theorem}\label{thm2}
	Given a ReLU function $\mathbf{y} = \mathrm{ReLU}(\mathbf{x})$ with an input set $\mathcal{X}=\bigcup\nolimits_{s=1}^{N}\mathcal{X}_s$ in which $\mathcal{X}_s$, $s = 1,\ldots,N$, is defined by (\ref{inputset1}), its output reachable set $\mathcal{Y}$ is
	\begin{equation}\label{thm2_1}
	\mathcal{Y} = \bigcup\nolimits_{s=1}^{N} \left(\mathcal{Y}^{+}_s \cup \mathcal{Y}^{-}_s \cup \bigcup\nolimits_{h=1}^{2^n-2}\mathcal{Y}_{h,s}\right)
	\end{equation}
	where  $\mathcal{Y}_{h,s}$, $h=1,\ldots,2^n-1$, $s =1,\ldots,N$, are as follows
	\begin{align*}
	\mathcal{Y}_s^{+} &= \left\{\mathbf{y} \mid \mathbf{A}_s \mathbf{y}\le \mathbf{b}_s \wedge \mathbf{y} > \mathbf{0} ,~\mathbf{y} \in \mathbb{R}^n \right\};
	\\
	\mathcal{Y}_s^ -  & = \left\{ {\begin{array}{*{20}c}
		{\{ \mathbf{0}\}, } & {\mathcal{X}_s^ -   \ne \emptyset }  \\
		\emptyset,  & {\mathcal{X}_s^ -  = \emptyset }  \\	
		\end{array} } \right.,
	\\
	\mathcal{X}_s^-& = \left\{\mathbf{x} \mid \left[ {\begin{array}{*{20}c}
		{\mathbf{A}_s }  \\
		{\mathbf{I} }  \\
		\end{array} } \right]\mathbf{x} \le \left[ {\begin{array}{*{20}c}
		{\mathbf{b}_s }  \\
		\mathbf{0}  \\
		\end{array} } \right],~\mathbf{x} \in \mathbb{R}^n\right\};
	\\
	\mathcal{Y}_{h,s} &=\left\{\mathbf{y} \mid \left[ {\begin{array}{*{20}c}
		{\mathbf{A}_s }  \\
		{\mathbf{I} - \mathbf{P}_h}  \\
		{ - \mathbf{P}_h }  \\
		\end{array} } \right]\mathbf{y} \le \left[ {\begin{array}{*{20}c}
		{\mathbf{b}_s }  \\
		\mathbf{0}  \\
		\mathbf{0}  \\	
		\end{array} } \right],~\mathbf{y} \in \mathbb{R}^n\right\}. \label{thm2_4}
	\end{align*}
\end{theorem}
\begin{proof}
	First, when $\mathbf{x} \in \mathcal{X}^{+}_s $, it has $\mathbf{y} = \mathbf{x}$, thus the output set $\mathcal{Y}_s^{+} = \mathcal{X}_s^{+}$, which is
	\begin{equation*} \label{pthm2_1}
	\mathcal{Y}_s^{+} = \left\{\mathbf{y} \mid \mathbf{A}_s \mathbf{y}\le \mathbf{b}_s \wedge \mathbf{y} > \mathbf{0} ,~\mathbf{y} \in \mathbb{R}^n \right\}.
	\end{equation*}
	
	Then, if $\mathbf{x} \in \mathcal{X}^{-}_s$, where $\mathcal{X}_s^-$ can be defined by
	\begin{equation*}
	\mathcal{X}_s^- = \left\{\mathbf{x} \mid \left[ {\begin{array}{*{20}c}
		{\mathbf{A}_s }  \\
		{\mathbf{I} }  \\
		\end{array} } \right]\mathbf{x} \le \left[ {\begin{array}{*{20}c}
		{\mathbf{b}_s }  \\
		\mathbf{0}  \\
		\end{array} } \right],~\mathbf{x} \in \mathbb{R}^n\right\}.
	\end{equation*}
	
	According to the definition of ReLU functions, it directly shows that the output set $\mathcal{Y}_s^-$ is
	\begin{equation*} \label{pthm2_3}
	\mathcal{Y}_s^ -   = \left\{ {\begin{array}{*{20}c}
		{\{ \mathbf{0}\}, } & {\mathcal{X}_s^ -   \ne \emptyset }  \\
		\emptyset,  & {\mathcal{X}_s^ -  = \emptyset }  \\	
		\end{array} } \right..
	\end{equation*}

	Finally, we consider $x \in \mathcal{X}_{h,s}^{-}$, $h=1,\ldots,2^{n}-2$, which can be expressed by
	\begin{equation*} \label{pthm2_6}
	\mathcal{X}_{h,s}^{-} = \left\{\mathbf{x} \mid
	\left[ {\begin{array}{*{20}c}
		{\mathbf{A}_s }  \\
		{\mathbf{I} - \mathbf{P}_h }  \\
		\end{array} } \right]\mathbf{x} \leq \left[ {\begin{array}{*{20}c}
		{\mathbf{b}_s }  \\
		\mathbf{0}  \\
		\end{array} } \right],~\mathbf{x} \in \mathbb{R}^n\right\}
	\end{equation*}
	where $h=1,\ldots,2^{n}-2$.
	
	Adding the additional constraint that $\mathbf{y} = \mathbf{P}_h \mathbf{x}\ge 0$, set $\mathcal{X}_{h,s}^+$, $h=1,\ldots,2^{n}-2$, is expressed as
	\begin{equation*} \label{pthm2_7}
	\mathcal{X}_{h,s}^{+} = \left\{\mathbf{x} \mid
	\left[ {\begin{array}{*{20}c}
		{\mathbf{A}_s }  \\
		{\mathbf{I} - \mathbf{P}_h }  \\
		{-\mathbf{P}_h }\\
		\end{array} } \right]\mathbf{x} \leq \left[ {\begin{array}{*{20}c}
		{\mathbf{b}_s }  \\
		\mathbf{0}  \\
		\mathbf{0} \\	
		\end{array} } \right],~\mathbf{x} \in \mathbb{R}^n\right\}
	\end{equation*}
	with $h=1,\ldots,2^{n}-2$.
	
	By the fact of $\mathbf{y} = \mathbf{x}$, $\mathbf{x} \in \mathcal{X}_{h,s}^{+}$, $h=1,\ldots,2^{n}-2$, it can be obtained that
	\begin{equation*} \label{pthm2_8}
	\mathcal{Y}_{h,s}=\mathcal{X}_{h,s}^{+},~h=1,\ldots,2^{n}-2.
	\end{equation*}
	
	Thus, based on Theorem \ref{thm1}, the output reachable set for input set $\mathcal{X}_s$ is
	\begin{equation*} \label{pthm2_9}
	\mathcal{Y}_s = \mathcal{Y}^{+}_s \cup \mathcal{Y}^{-}_s \cup \left(\bigcup\nolimits_{h=1}^{2^n-2}\mathcal{Y}_{h,s}\right).
	\end{equation*}
	
	Moreover, for input set $\mathcal{X}=\bigcup\nolimits_{s=1}^{N}\mathcal{X}_s$,  the output reachable set is $\mathcal{Y} = \bigcup\nolimits_{s=1}^{N}\mathcal{Y}_s$, which implies (\ref{thm2_1}) holds. 
\end{proof}

According to the result in Theorem \ref{thm2}, the following useful corollary can be derived before ending this subsection.

\begin{corollary}\label{cor1}
	Given a ReLU function $\mathbf{y} = \mathrm{ReLU}(\mathbf{x})$, if the input set $\mathcal{X}$ is a union of polyhedra, then the output reachable set $\mathcal{Y}$  is also a union of polyhedra.
\end{corollary}
\begin{proof}
	By Theorem \ref{thm2}, $\mathcal{Y}^+_s$, $\mathcal{Y}^-_s$, $\mathcal{Y}_{h,s}$, $h=1,\ldots,2^n-1$, $s=1,\ldots,N$, are all defined as polyhedra when input set $\mathcal{X}$ is a union of polyhedra, thus $\mathcal{Y} = \bigcup\nolimits_{s=1}^{N} \left(\mathcal{Y}^{+}_s \cup \mathcal{Y}^{-}_s \cup \left( \bigcup\nolimits_{h=1}^{2^n-2}\mathcal{Y}_{h,s}\right)\right)$ is a union of polyhedra. The proof is complete.
\end{proof}

Theorems \ref{thm1} and \ref{thm2} present the output reach set of a ReLU function. Following those results for ReLU functions, the reachable set for neural networks composed of  ReLU  activations will be studied in the next section.

\section{Reach Set Computation and Verification for ReLU Neural Networks}
Based on the reachability analysis for ReLU activation functions in previous section, we are ready for the main problems in this paper, the reachable set computation and verification problems for ReLU neural networks. First, ReLU neural networks can be expressed recursively  in a layer-by-layer form of
\begin{align}\label{ReLU_NN}
\left\{ {\begin{array}{*{20}l}
	{\mathbf{y}^{[\ell]}=\mathrm{ReLU}(\mathbf{W}^{[\ell]}\mathbf{x}^{[\ell]}+\boldsymbol{\theta}^{[\ell]}),~\ell = 1,\ldots,L}  \\
	{\mathbf{x}^{[\ell+1]} = \mathbf{y}^{[\ell]},~\ell = 1,\ldots,L-1}  \\
	\mathbf{x}^{[1]} = \mathbf{x}^{[0]}
	\end{array} } \right.
\end{align}
where $\mathbf{x}^{[0]} \in \mathcal{X}^{[0]} \subset \mathbb{R}^{n^{[0]}}$ is the input set defined by a union of polyhedra as in (\ref{inputset}).
The input set and output set of layer $\ell$ are denoted as $\mathcal{X}^{[\ell]}$ and $\mathcal{Y}^{[\ell]}$, respectively.

\begin{lemma}\label{lem1}
	Consider neural network (\ref{ReLU_NN}) with input set $\mathcal{X}^{[0]}$ defined by a union of polyhedra (\ref{inputset}), the output sets of each layer $\mathcal{Y}^{[\ell]}$, $\ell = 1, \ldots,L$, are all defined by a union of polyhedra.
\end{lemma}
\begin{proof}
	Consider input set of layer $\ell=1$,  $\mathcal{X}^{[1]}=\mathcal{X}^{[0]}$ which is a union of polyhedra, the following set which is an affine map of $\mathcal{X}^{[1]}$ should be a union of polyhedra
	\begin{equation*}
	\tilde{\mathcal{X}}^{[1]} = \left\{\tilde{\mathbf{x}}^{[1]} \mid \tilde{\mathbf{x}}^{[1]}=\mathbf{W}^{[1]}\mathbf{x}^{[1]}+\boldsymbol{\theta}^{[1]},~\mathbf{x} \in \mathcal{X}^{[1]}\right\}.
	\end{equation*}
	
	Then, using Corollary \ref{cor1} and $\mathbf{y}^{[1]} = \mathrm{ReLU}(\tilde{\mathbf{x}}^{[1]})$, the output reachable set $\mathcal{Y}^{[1]}$ is a union of polyhedra.
	
	Also from (\ref{ReLU_NN}), because we have
	\begin{equation*}
	\mathcal{X}^{[\ell+1]} = \mathcal{Y}^{[\ell]},~~\ell = 1,\ldots,L-1
	\end{equation*}
	The above procedure can be iterated from $\ell=1$ to $\ell = L$ to claim $\mathcal{Y}^{[\ell]}$, $\ell=1,\ldots,L$, are all defined by a union of polyhedra. 
\end{proof}

By Lemma \ref{lem1}, the output sets of each layer are defined as a union of polyhedra, and due to $\mathcal{X}^{[\ell+1]} = \mathcal{Y}^{[\ell]}$, the input set of each layer can be represented as $\mathcal{X}^{[\ell]}= \bigcup_{s=1}^{N_{\ell}}\mathcal{X}_s^{[\ell]}$, in which  $\mathcal{X}_s^{[\ell]}$ is
\begin{equation}\label{layer_input}
\mathcal{X}^{[\ell]}_s \triangleq\left\{ \mathbf{x}^{[\ell]} \mid  \mathbf{A}_s^{[\ell]}\mathbf{x}^{[\ell]} \le \mathbf{b}_s^{[\ell]},~\mathbf{x} \in \mathbb{R}^{n^{[\ell]}}\right\}.
\end{equation}

With regard to the input set of layer $\ell$ described by (\ref{layer_input}), the output reachable set can be obtained by the following theorem, which is the main result in this paper.
\begin{theorem}
	Consider layer $\ell$ of ReLU neural network (\ref{ReLU_NN}) with input set $\mathcal{X}^{[\ell]}$ defined by (\ref{layer_input}), the output reachable set $\mathcal{Y}^{[\ell]}$ of layer $\ell$  is
	\begin{equation} \label{thm3_1}
	\mathcal{Y}^{[\ell]} = \bigcup\nolimits_{s=1}^{N_{\ell}} \left(\mathcal{Y}^{[\ell]+}_s \cup \mathcal{Y}^{[\ell]-}_s \cup\left( \bigcup\nolimits_{h=1}^{2^{n^{[\ell]}}-2}\mathcal{Y}_{h,s}^{[\ell]}\right)\right)
	\end{equation}
	where
	\begin{align*}
	\mathcal{Y}_s^{[\ell]+}& = \left\{\mathbf{y} \mid \mathbf{y} =\mathbf{W}^{[\ell]}\mathbf{x}^{[\ell]}+\boldsymbol{\theta}^{[\ell]},~\mathbf{x} \in \mathcal{X}_s^{[\ell]+}\right\},
	\\
	\mathcal{X}_s^{[\ell]+} &= \left\{\mathbf{x}^{[\ell]} \mid \mathbf{A}_s^{[\ell]} \mathbf{x}^{[\ell]}\le \mathbf{b}_s^{[\ell]} \wedge \mathbf{W}^{[\ell]}\mathbf{x}^{[\ell]} > -\boldsymbol{\theta}^{[\ell]} ,~\mathbf{x}^{[\ell]} \in \mathbb{R}^{n^{[\ell]}} \right\};
	\\
	\mathcal{Y}_s^{[\ell]-}  & = \left\{ {\begin{array}{*{20}c}
		{\{ \mathbf{0}\}, } & {\mathcal{X}_s^ {[\ell]-}   \ne \emptyset }  \\
		\emptyset,  & {\mathcal{X}_s^ {[\ell]-}  = \emptyset }  \\	
		\end{array} } \right.,
	\\
	\mathcal{X}_s^{[\ell]-}& = \left\{\mathbf{x}^{[\ell]} \mid \left[ {\begin{array}{*{20}c}
		{\mathbf{A}_s^{[\ell]} }  \\
		{\mathbf{\mathbf{W}^{[\ell]}} }  \\
		\end{array} } \right]\mathbf{x}^{[\ell]} \le \left[ {\begin{array}{*{20}c}
		{\mathbf{b}_s^{[\ell]} }  \\
		\mathbf{-\boldsymbol{\theta}}^{[\ell]}  \\
		\end{array} } \right],~\mathbf{x}^{[\ell]} \in \mathbb{R}^{n^{[\ell]}}\right\};
	\\
	\mathcal{Y}_{h,s}^{[\ell]}& = \left\{\mathbf{y} \mid \mathbf{y} =\mathbf{W}^{[\ell]}\mathbf{x}^{[\ell]}+\boldsymbol{\theta}^{[\ell]},~\mathbf{x} \in \mathcal{X}_{h,s}^{[\ell]}\right\},
	\\
	\mathcal{X}_{h,s}^{[\ell]} &=\left\{\mathbf{x}^{[\ell]} \mid \tilde{\mathbf{A}}_{h,s}^{[\ell]}\mathbf{x}^{[\ell]} \le \tilde{\mathbf{b}}_{h,s}^{[\ell]} ,~\mathbf{x}^{[\ell]} \in \mathbb{R}^{n^{[\ell]}}\right\},
	\\
	\tilde{\mathbf{A}}_{h,s}^{[\ell]}&=\left[ {\begin{array}{*{20}c}
		{\mathbf{A}_s^{[\ell]} }  \\
		{(\mathbf{I} - \mathbf{P}_h)\mathbf{W}^{[\ell]}}  \\
		{ - \mathbf{P}_h \mathbf{W}^{[\ell]}}  \\
		\end{array} } \right],~\tilde{\mathbf{b}}_{h,s}^{[\ell]}=\left[ {\begin{array}{*{20}c}
		{\mathbf{b}_s^{[\ell]} }  \\
		( \mathbf{P}_h-\mathbf{I})\boldsymbol{\theta}^{[\ell]}  \\
		\mathbf{P}_h\boldsymbol{\theta}^{[\ell]}  \\	
		\end{array} } \right].
	\end{align*}
\end{theorem}
\begin{proof}
	The proof is briefly presented below since it is similar to the proof line for Theorem \ref{thm2}.
	
	Like the proof line in Theorem \ref{thm2}, when $\mathbf{y} > 0$ which means $\mathbf{W}^{[\ell]}\mathbf{x}^{[\ell]} + \boldsymbol{\theta}^{[\ell]} < 0$, combining $\mathbf{A}_s^{[\ell]}\mathbf{x}^{[\ell]}\le \mathbf{b}^{[\ell]}_s$, it defines a set  $\mathcal{X}_s^{[\ell]+}$ as
	\begin{equation*}
	\mathcal{X}_s^{[\ell]+} = \left\{\mathbf{x}^{[\ell]} \mid \mathbf{A}_s^{[\ell]} \mathbf{x}^{[\ell]}\le \mathbf{b}_s^{[\ell]} \wedge \mathbf{W}^{[\ell]}\mathbf{x}^{[\ell]} > -\boldsymbol{\theta}^{[\ell]} ,\mathbf{x}^{[\ell]} \in \mathbb{R}^{n^{[\ell]}} \right\}.
	\end{equation*}
	
	Moreover, due to $\mathbf{y}^{[\ell]} = \mathbf{W}^{[\ell]}\mathbf{x}^{[\ell]}+\boldsymbol{\theta}^{[\ell]}$, the output reachable set for $\mathbf{y}^{[\ell]}$ is
	\begin{equation*}
	\mathcal{Y}_s^{[\ell]+} = \left\{\mathbf{y} \mid \mathbf{y} =\mathbf{W}^{[\ell]} \mathbf{x}^{[\ell]}+\boldsymbol{\theta}^{[\ell]},~\mathbf{x} \in \mathcal{X}_s^{[\ell]+}\right\}.
	\end{equation*}
	
	Similarly, if $\mathbf{y} \le 0$, we have set $\mathcal{X}_s^{[\ell]-}$ defined by
	\begin{equation*}
	\mathcal{X}_s^{[\ell]-} = \left\{\mathbf{x}^{[\ell]} \mid \left[ {\begin{array}{*{20}c}
		{\mathbf{A}_s^{[\ell]} }  \\
		{\mathbf{\mathbf{W}^{[\ell]}} }  \\
		\end{array} } \right]\mathbf{x}^{[\ell]} \le \left[ {\begin{array}{*{20}c}
		{\mathbf{b}_s^{[\ell]} }  \\
		\mathbf{-\boldsymbol{\theta}}^{[\ell]}  \\
		\end{array} } \right],~\mathbf{x}^{[\ell]} \in \mathbb{R}^{n^{[\ell]}}\right\}.
	\end{equation*}
	
	Using ReLU function to get
	\begin{equation*}
	\mathcal{Y}_s^ {[\ell]-}   = \left\{ {\begin{array}{*{20}c}
		{\{ \mathbf{0}\}, } & {\mathcal{X}_s^ {[\ell]-}   \ne \emptyset }  \\
		\emptyset,  & {\mathcal{X}_s^  {[\ell]-}  = \emptyset }  \\	
		\end{array} } \right..
	\end{equation*}

	Lastly, we consider the case that $\mathbf{y}$ has both positive and non-positive elements, that is 	 	
	$\mathbf{y} \in \mathcal{Y}_{h,s}^{[\ell]-}$, $h=1,\ldots,2^n-2$, where $\mathcal{Y}_{h,s}^{[\ell]-}$ is expressed as follows
	\begin{align*}
	\mathcal{Y}_{h,s}^{[\ell]} &= \left\{\mathbf{y} \mid \mathbf{y} =\mathbf{W}^{[\ell]}\mathbf{x}^{[\ell]}+\boldsymbol{\theta}^{[\ell]},~\mathbf{x} \in \mathcal{X}_{h,s}^{[\ell]-}\right\},
	\\
	\mathcal{X}_{h,s}^{[\ell]-}&= \left\{\mathbf{x}^{[\ell]} \mid \hat{\mathbf{A}}_{h,s}^{[\ell]}\mathbf{x}^{[\ell]} \le \hat{\mathbf{b}}_{h,s}^{[\ell]},~\mathbf{x}^{[\ell]} \in \mathbb{R}^{n^{[\ell]}}\right\}.
	\end{align*}
	where 
	\begin{equation*}
	\hat{\mathbf{A}}_{h,s}^{[\ell]}=\left[ {\begin{array}{*{20}c}
		{\mathbf{A}_s^{[\ell]} }  \\
		{(\mathbf{I} - \mathbf{P}_h)\mathbf{W}^{[\ell]}}  \\
		\end{array} } \right],~ \hat{\mathbf{b}}_{h,s}^{[\ell]} = \left[ {\begin{array}{*{20}c}
		{\mathbf{b}_s^{[\ell]} }  \\
		( \mathbf{P}_h-\mathbf{I})\boldsymbol{\theta}^{[\ell]}  \\	
		\end{array} } \right].
	\end{equation*}
	
	Furthermore, due to $\mathbf{y} \ge 0$,  an additional constraint $\mathbf{P}_h(\mathbf{W}^{[\ell]}\mathbf{x}^{[\ell]}+\boldsymbol{\theta}^{[\ell]})\ge 0$ should be added to $\mathcal{X}_{h,s}^{[\ell]-}$ to obtain $\mathcal{X}_{h,s}^{[\ell]}$ as
	\begin{equation*}
	\mathcal{X}_{h,s}^{[\ell]} =\left\{\mathbf{x}^{[\ell]} \mid \tilde{\mathbf{A}}_{h,s}^{[\ell]} \mathbf{x}^{[\ell]} \le \tilde{\mathbf{b}}_{h,s}^{[\ell]} ,~\mathbf{x}^{[\ell]} \in \mathbb{R}^{n^{[\ell]}}\right\}
	\end{equation*}
	where 
	\begin{equation*}
	\tilde{\mathbf{A}}_{h,s}^{[\ell]} = \left[ {\begin{array}{*{20}c}
		{	\hat {\mathbf{A}}_{h,s}^{[\ell]} }  \\
		{ - \mathbf{P}_h \mathbf{W}^{[\ell]}}  \\
		\end{array} } \right],~	\tilde{\mathbf{b}}_{h,s}^{[\ell]}=\left[ {\begin{array}{*{20}c}
		{	\hat{\mathbf{b}}_{h,s}^{[\ell]} }   \\
		\mathbf{P}_h\boldsymbol{\theta}^{[\ell]}  \\	
		\end{array} } \right].
	\end{equation*}

	Then, following the guidelines of Theorem \ref{thm2}, the output reachable set of the form (\ref{thm3_1}) can be established. 
\end{proof}

As for linear activations, which are commonly used in the output layer, the output reachable set can be computed like the set $\mathcal{Y}_s^{[\ell]+}$ for ReLU, without the constraint $\mathbf{y} > 0$. The following corollary is given  for  linear layers.

\begin{corollary}
	Consider a linear layer with input set $\mathcal{X}^{[\ell]}$ defined by (\ref{layer_input}), the output reachable set $\mathcal{Y}^{[\ell]}$ of linear layer $\ell$ is
	\begin{equation}
	\mathcal{Y}^{[\ell]} = \bigcup\nolimits_{s=1}^{N_{\ell}}\mathcal{Y}_{s}^{[\ell]}
	\end{equation}
	where
	$
	\mathcal{Y}_s^{[\ell]} = \left\{\mathbf{y} \mid \mathbf{y} =\mathbf{W}^{[\ell]}\mathbf{x}^{[\ell]}+\boldsymbol{\theta}^{[\ell]},~\mathbf{x} \in \mathcal{X}_s^{[\ell]}\right\}
	$.
\end{corollary}
\begin{proof}
	For an input $\mathbf{x} \in \mathcal{X}_s^{[\ell]}$, the linear relation $ \mathbf{y} =\mathbf{W}^{[\ell]}\mathbf{x}^{[\ell]}+\boldsymbol{\theta}^{[\ell]}$ implies that the output reachable set is 	$
	\mathcal{Y}_s^{[\ell]} = \left\{\mathbf{y} \mid \mathbf{y} =\mathbf{W}^{[\ell]}\mathbf{x}^{[\ell]}+\boldsymbol{\theta}^{[\ell]},~\mathbf{x} \in \mathcal{X}_s^{[\ell]}\right\}
	$. Moreover, because of $\mathcal{X}^{[\ell]} = \bigcup\nolimits_{s=1}^{N_{\ell}}\mathcal{X}_{s}^{[\ell]}$, it directly leads to $\mathcal{Y}^{[\ell]} = \bigcup\nolimits_{s=1}^{N_{\ell}}\mathcal{Y}_{s}^{[\ell]}$. 
\end{proof}

With the output reachable set computation results for both ReLU layers and linear layers, we are now ready to present the output reachable set computation results along with safety verification results summarized as functions \texttt{OutputReLU},  \texttt{OutputReLUNetwork} and \texttt{VeriReLUNetwork} presented in Algorithms \ref{algorithm_1}, \ref{algorithm_2} and \ref{algorithm_3}, respectively.

\begin{algorithm}
	\caption{Output Reach Set Computation for ReLU Layers} \label{algorithm_1}
	
	\begin{algorithmic}[1]
		\Require ReLU neural network weight matrix and bias $\langle\mathbf{W}^{[\ell]},~\boldsymbol{\theta}^{[\ell]}\rangle$, input set $\mathcal{X}^{[\ell]}= \bigcup_{s=1}^{N_{\ell}}\mathcal{X}_s^{[\ell]}$ with $\mathcal{X}^{[\ell]}_s \triangleq\{ \mathbf{x}^{[\ell]} \mid  \mathbf{A}_s^{[\ell]}\mathbf{x}^{[\ell]} \le \mathbf{b}_s^{[\ell]},~\mathbf{x} \in \mathbb{R}^{n^{[\ell]}}\} $.
		\Ensure Output reachable set $\mathcal{Y}^{[\ell]}$.
		
		\Function{OutputReLU}{$\langle\mathbf{W}^{[\ell]},~\boldsymbol{\theta}^{[\ell]}\rangle$,~$\mathcal{X}^{[\ell]}$}
		\For{$s=1:1:N_{\ell}$}
		\State $\mathcal{X}_s^{[\ell]+} \gets \left\{\mathbf{x}^{[\ell]} \mid \mathbf{A}_s^{[\ell]} \mathbf{x}^{[\ell]}\le \mathbf{b}_s^{[\ell]} \wedge \mathbf{W}^{[\ell]}\mathbf{x}^{[\ell]} > -\boldsymbol{\theta}^{[\ell]} \right\}$
		
		\State 	$\mathcal{Y}_s^{[\ell]+} \gets \left\{\mathbf{y} \mid \mathbf{y} =\mathbf{W}^{[\ell]}\mathbf{x}^{[\ell]}+\boldsymbol{\theta}^{[\ell]},~\mathbf{x} \in \mathcal{X}_s^{[\ell]+}\right\}$
		
		\State $\mathcal{X}_s^{[\ell]-} \gets \left\{\mathbf{x}^{[\ell]} \mid \left[ {\begin{array}{*{20}c}
			{\mathbf{A}_s^{[\ell]} }  \\
			{\mathbf{\mathbf{W}^{[\ell]}} }  \\
			\end{array} } \right]\mathbf{x}^{[\ell]} \le \left[ {\begin{array}{*{20}c}
			{\mathbf{b}_s^{[\ell]} }  \\
			\mathbf{-\boldsymbol{\theta}}^{[\ell]}  \\
			\end{array} } \right]\right\}$
		
		\If{$\mathcal{X}_s^{[\ell]-} \ne \emptyset$}
		\State $\mathcal{Y}_s^{[\ell]-} \gets \{\mathbf{0}\} $
		\Else
		\State $\mathcal{Y}_s^{[\ell]-} \gets \emptyset$
		\EndIf
		
		\For{$h=1:1:2^{n^{[\ell]}}-2$}
		\State $	\tilde{\mathbf{A}}_{h,s}^{[\ell]} \gets\left[ {\begin{array}{*{20}c}
			{\mathbf{A}_s^{[\ell]} }  \\
			{(\mathbf{I} - \mathbf{P}_h)\mathbf{W}^{[\ell]}}  \\
			{ - \mathbf{P}_h \mathbf{W}^{[\ell]}}  \\
			\end{array} } \right]$
		\State $\tilde{\mathbf{b}}_{h,s}^{[\ell]}\gets\left[ {\begin{array}{*{20}c}
			{\mathbf{b}_s^{[\ell]} }  \\
			( \mathbf{P}_h-\mathbf{I})\boldsymbol{\theta}^{[\ell]}  \\
			\mathbf{P}_h\boldsymbol{\theta}^{[\ell]}  \\	
			\end{array} } \right]$
		\State $\mathcal{X}_{h,s}^{[\ell]} \gets\left\{\mathbf{x}^{[\ell]} \mid \tilde{\mathbf{A}}_{h,s}^{[\ell]}\mathbf{x}^{[\ell]} \le \tilde{\mathbf{b}}_{h,s}^{[\ell]},~\mathbf{x}^{[\ell]} \in \mathbb{R}^{n^{[\ell]}}\right\} $
		\State  $\mathcal{Y}_{h,s}^{[\ell]} \gets \left\{\mathbf{y} \mid \mathbf{y} =\mathbf{W}^{[\ell]}\mathbf{x}^{[\ell]}+\boldsymbol{\theta}^{[\ell]},~\mathbf{x} \in \mathcal{X}_{h,s}^{[\ell]}\right\}$
		
		\EndFor
		\State $\mathcal{Y}^{[\ell]}_s \gets \mathcal{Y}^{[\ell]+}_s \cup \mathcal{Y}^{[\ell]-}_s \cup\left( \bigcup\nolimits_{h=1}^{2^{n^{[\ell]}}-2}\mathcal{Y}_{h,s}^{[\ell]}\right)$
		\EndFor	
		
		\State \Return $\mathcal{Y}^{[\ell]} \gets \bigcup\nolimits_{s=1}^{N_{\ell}}\mathcal{Y}^{[\ell]}_s$
		
		\EndFunction
	\end{algorithmic}
\end{algorithm}

\begin{algorithm}
	\caption{Output Reach Set Computation for ReLU Networks} \label{algorithm_2}
	
	\begin{algorithmic}[1]
		\Require ReLU neural network weight matrices and biases $\langle\mathbf{W}^{[\ell]},\boldsymbol{\theta}^{[\ell]}\rangle$, $,\ell = 1,\ldots,L$, input set $\mathcal{X}^{[0]}= \bigcup_{s=1}^{N_{0}}\mathcal{X}_s^{[0]}$ with $\mathcal{X}^{[0]}_s \triangleq\{ \mathbf{x}^{[0]} \mid  \mathbf{A}_s^{[0]}\mathbf{x}^{[0]} \le \mathbf{b}_s^{[0]},~\mathbf{x} \in \mathbb{R}^{n^{[0]}}\} $.
		\Ensure Output reachable set $\mathcal{Y}^{[L]}$.
		
		\Function{OutputReLUNetwork}{$\langle\mathbf{W}^{[\ell]},~\boldsymbol{\theta}^{[\ell]}\rangle$,~$\ell = 1,\ldots,L$,~$\mathcal{X}^{[0]}$}
		\State $\mathcal{X}^{[1]} \gets \mathcal{X}^{[0]}$
		\For{$\ell=1:1:L$}
		\If{Layer $\ell$ is a linear layer}
		
		\State $\mathcal{Y}_s^{[\ell]} \gets \left\{\mathbf{y} \mid \mathbf{y} =\mathbf{W}^{[\ell]}\mathbf{x}^{[\ell]}+\boldsymbol{\theta}^{[\ell]},~\mathbf{x} \in \mathcal{X}_s^{[\ell]}\right\}$
		
		\State 	$\mathcal{Y}^{[\ell]} \gets \bigcup\nolimits_{s=1}^{N_{\ell}}\mathcal{Y}_{s}^{[\ell]}$
		
		\ElsIf{Layer $\ell$ is a ReLU Layer}

		\State $\mathcal{Y}^{[\ell]} \gets\mathrm{OutputReLU}(\mathbf{W}^{[\ell]},\boldsymbol{\theta}^{[\ell]},\mathcal{X}^{[\ell]})$
		\EndIf
		
		\State $\mathcal{X}^{[\ell+1]} \gets \mathcal{Y}^{[\ell]}$
		
		\EndFor
		
		\State \Return $\mathcal{Y}^{[L]}$
		
		\EndFunction
	\end{algorithmic}
\end{algorithm}

\begin{algorithm}
	\caption{Safety Verification for ReLU Networks} \label{algorithm_3}
	
	\begin{algorithmic}[1]
		\Require ReLU neural network weight matrices weight matrices and biases $\langle\mathbf{W}^{[\ell]},\boldsymbol{\theta}^{[\ell]}\rangle$, $\ell = 1,\ldots,L$, input set $\mathcal{X}^{[0]}= \bigcup_{s=1}^{N_{0}}\mathcal{X}_s^{[0]}$ with $\mathcal{X}^{[0]}_s \triangleq\{ \mathbf{x}^{[0]} \mid  \mathbf{A}_s^{[0]}\mathbf{x}^{[0]} \le \mathbf{b}_s^{[0]},~\mathbf{x} \in \mathbb{R}^{n^{[0]}}\ $, safety specification $\mathcal{S} = \bigcup_{m =1}^{M}{\mathcal{S}_m}$ with $\mathcal{S}_m\triangleq\{\mathbf{y}^{[L]}\mid \mathbf{C}_m \mathbf{y}^{[L]} \le \mathbf{d}_m,~\mathbf{y} \in \mathbb{R}^{n^{[L]}}\}$.
		
		\Ensure SAFE or UNSAFE.
		
		\Function{VeriReLUNetwork}{$\langle\mathbf{W}^{[\ell]},~\boldsymbol{\theta}^{[\ell]}\rangle,~\ell=1,\ldots,L$,~$\mathcal{X}^{[0]}$,~$\mathcal{S}$}
		
		\State $\mathcal{Y}^{[L]} = \mathrm{OutputReLUNetwork}(\langle\mathbf{W}^{[\ell]},~\boldsymbol{\theta}^{[\ell]}\rangle,~\ell=1,\ldots,L,~\mathcal{X}^{[0]})$
		
		\If{$\mathcal{Y}^{[L]} \cap \neg \mathcal{S} = \emptyset$}
		\State \Return SAFE
		
		\Else
		
		\State \Return UNSAFE

		\EndIf	
		\EndFunction
	\end{algorithmic}
\end{algorithm}	

Some remarks are given below for those algorithms:  
\begin{enumerate}
	\item[(1)] Algorithm \ref{algorithm_1} is for the output reachable set computation for ReLU layers, named function \texttt{OutputReLU}, which is responsible to compute the output of each single layer equipped with ReLU activations. The input of function \texttt{OutputReLU} are the weight matrix and bias $\langle\mathbf{W}^{[\ell]},\boldsymbol{\theta}^{[\ell]}\rangle$, and layer input set $\mathcal{X}^{[\ell]}$. The output is the output reachable set $\mathcal{Y}^{[\ell]}$ of layer $\ell$.
	
	\item[(2)] Algorithm \ref{algorithm_2} is formulated based on function \texttt{OutputReLU} presented in Algorithm \ref{algorithm_1}. The function is called \texttt{OutputReLUNetwork}, whose input is the weight matrices and biases of neural network, $\langle\mathbf{W}^{[\ell]},~\boldsymbol{\theta}^{[\ell]}\rangle$,~$\ell = 1,\ldots,L$, and input set $\mathcal{X}^{[0]}$, and output is the output reachable set $\mathcal{Y}^{[L]}$ of the ReLU neural network. It is used for computing the output reachable set of a ReLU neural network. It should be noted that the linear layer is also allowed to exist in the neural network. Moreover, it worth mentioning that the output reachable set computed by function \texttt{OutputReLUNetwork} is the exact output reachable set, not an over-approximation.
	
	\item[(3)] Using the output reachable set computed in function \texttt{OutputReLUNetwork}, the safety verification for ReLU neural networks can be converted to check if there exists non-empty intersection between output reachable set and unsafe regions, which is summarized as function \texttt{VeriReLUNetwork}. The input of function \texttt{VeriReLUNetwork} is  weight matrices and biases $\langle\mathbf{W}^{[\ell]},~\boldsymbol{\theta}^{[\ell]}\rangle$,~$\ell = 1,\ldots,L$, input set $\mathcal{X}^{[0]}$ and safety specification $\mathcal{S}$. The output is to claim the neural network safe or unsafe.  Algorithm \ref{algorithm_3} is sound for the cases of SAFE and UNSAFE, that is, if it returns SAFE then the system is safe; when it returns UNSAFE there exists at least one output from input set is unsafe.
\end{enumerate}

The computation involved in the developed algorithms are operations of polyhedra, thus the computational complexity closely relates to the number of polyhedra generated and involved in the operations. For Algorithm \ref{algorithm_1}, the number of polyhedra contained in the output reachable set $\mathcal{Y}^{[\ell]}$ for the single layer $\ell$ is $ N_{\ell}2^{n^{[\ell]}}
$. Furthermore, considering a ReLU neural network with $L$ layers, the number of polyhedra for the $\ell$th layer is $ N_{0}\prod\nolimits_{s = 1}^{\ell} {2^{n^{[s]} } } $. It can be observed that the number of polyhedra increases exponentially as the number of layer grows, that also indicates that the neural networks with multiple layers have great ability to approximate complex functions and resolve complicate problems. The output of a ReLU network consisting of $L$ layers contains $ N_{0}\prod\nolimits_{s = 1}^{L} {2^{n^{[s]} } } $ polyhedra. The numbers for a single layer, the layer in a network and the whole neural network are listed in Table \ref{tab1}. These numbers represent the maximal numbers of polyhedra that may be obtained in the computation process, however in practice, the number is usually smaller than them since empty sets are produced for the output of the layers during the computation procedure. 

\begin{table}
	\centering
	\caption{Number of Polyhedra in Reachable Set Computation for ReLU Networks}\label{tab1}
	\begin{tabular}{|c|c|c|}
		\hline
		Single Layer $\ell$ & $\ell$th Layer & Neural Network\\
		\hline
		$ N_{\ell}2^{n^{[\ell]}}$ & $ N_{0}\prod\nolimits_{s = 1}^{\ell} {2^{n^{[s]} } } $ &  $ N_{0}\prod\nolimits_{s = 1}^{L} {2^{n^{[s]} } } $  \\
		\hline
		
	\end{tabular}
\end{table} 

In addition, it can be observed that the operations for computing $\mathcal{Y}_s^{[\ell]+}$, $\mathcal{Y}_s^{[\ell]-}$ and $\mathcal{Y}_{h,s}^{[\ell]}$ are all independent, thus those computations in Algorithm \ref{algorithm_1} can be executed in a parallel manner to enhance the scalability of the proposed approach. In the following example, it will be shown how the computational time will be reduced by employing parallel computing techniques.

\section{Numerical Example}
Consider a neural network with 3 inputs, 2 outputs and 7 hidden layers, see Figure \ref{NN3_7_2}. Each layer has 7 neurons, the weights and biases are all randomly generated. Hidden layers are chosen as ReLUs and the output layer is with linear functions. The example is
computed by using Matlab and Multi-parametric toolbox 3.0 \cite{herceg2013multi} on a personal computer with Windows 7, Intel Core i5-4200U,
1.6GHz, 4 GB RAM.
The first objective in this example is to compute the output reachable set of this ReLU neural network with respect to polytopic input sets.
\begin{figure}[h!]
	\begin{center}
		\includegraphics[width=12cm]{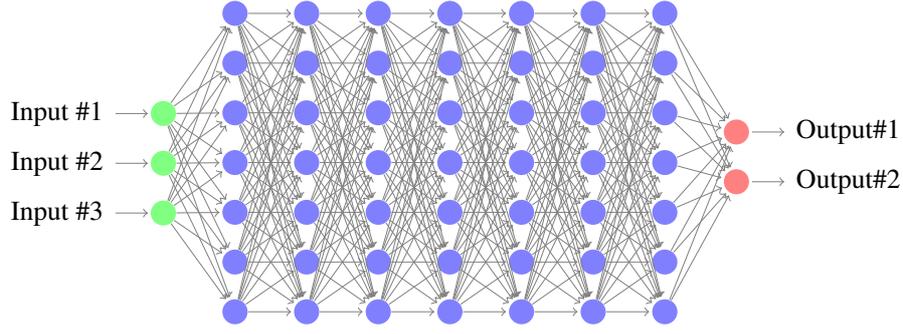}
		\caption{Neural network with 3 inputs, 2 outputs and 7 hidden layers of 7 neurons. The hidden layers are with ReLU activation functions and output layer is with linear functions. The weight matrices and bias vectors for the neural network in the example are randomly generated.}
		\label{NN3_7_2}
	\end{center}
\end{figure}

The input is assumed to be the $\mathcal{X}^{[0]} \triangleq \{\mathbf{x} \mid \left\|\mathbf{x}\right\|_{\infty}\le 1,\mathbf{x} \in \mathbb{R}^{3}\}$. Using function \texttt{OutputReLUNetwork} in Algorithm \ref{algorithm_2}, the computed reachable set can be obtained, which is a union of 1250 polyhedra, as shown in Figure \ref{fig:reachset}. Then, to validate the result, we discretize the input set $\mathcal{X}^{[0]}$ by step 0.05 and generate 8000 outputs in total for each discretizing point. As shown in Figure \ref{fig:output}, it can be observed that all the generated outputs are located in the computed reachable set to validate our approach.
\begin{figure}[h!]
	\begin{center}
		\includegraphics[width=12cm]{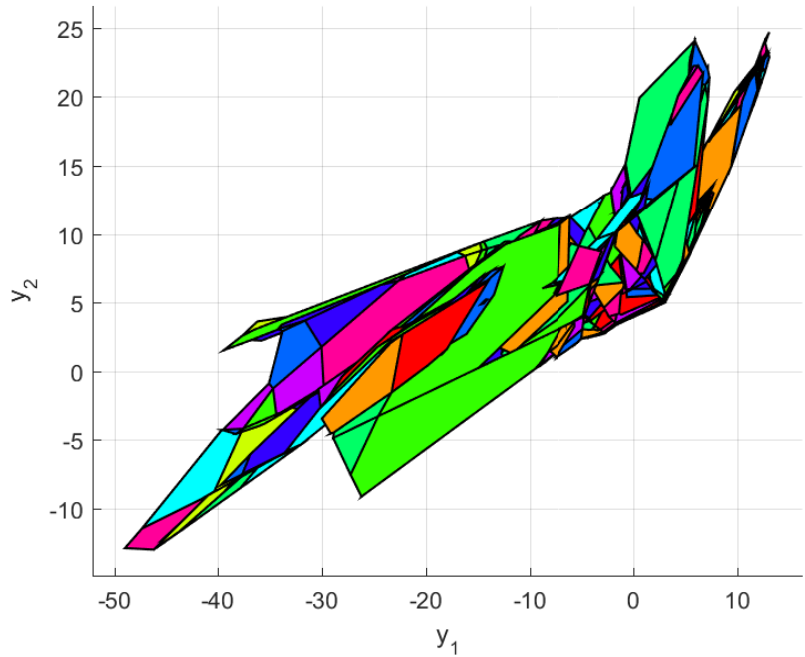}
		\caption{Given an input set $\mathcal{X}^{[0]} \triangleq \{\mathbf{x} \mid \left\|\mathbf{x}\right\|_{\infty}\le 1,\mathbf{x} \in \mathbb{R}^{3}\}$, the output reachable set of the proposed neural network is computed out and described by a union of 1250 polyhedra.}
		\label{fig:reachset}
	\end{center}
\end{figure}

\begin{figure}[h!]
	\begin{center}
		\includegraphics[width=12cm]{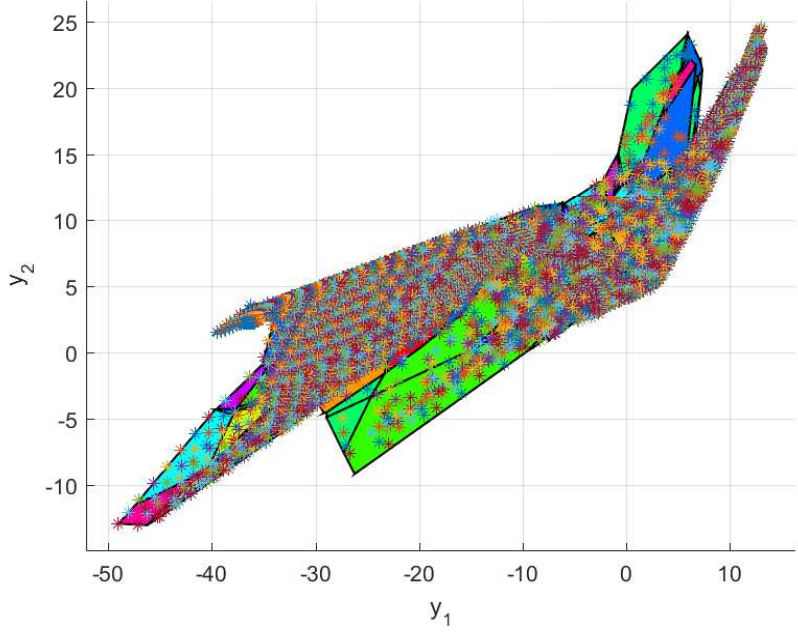}
		\caption{Discretizing the input set $\mathcal{X}^{[0]}$ by step 0.05, we generate 8000 outputs from those discretized points. The outputs are denoted by marker $*$, it can be seen that all the outputs are located in the computed reachable set.}
		\label{fig:output}
	\end{center}
\end{figure}

\begin{figure}[h!]
	\begin{center}
		\includegraphics[width=12cm]{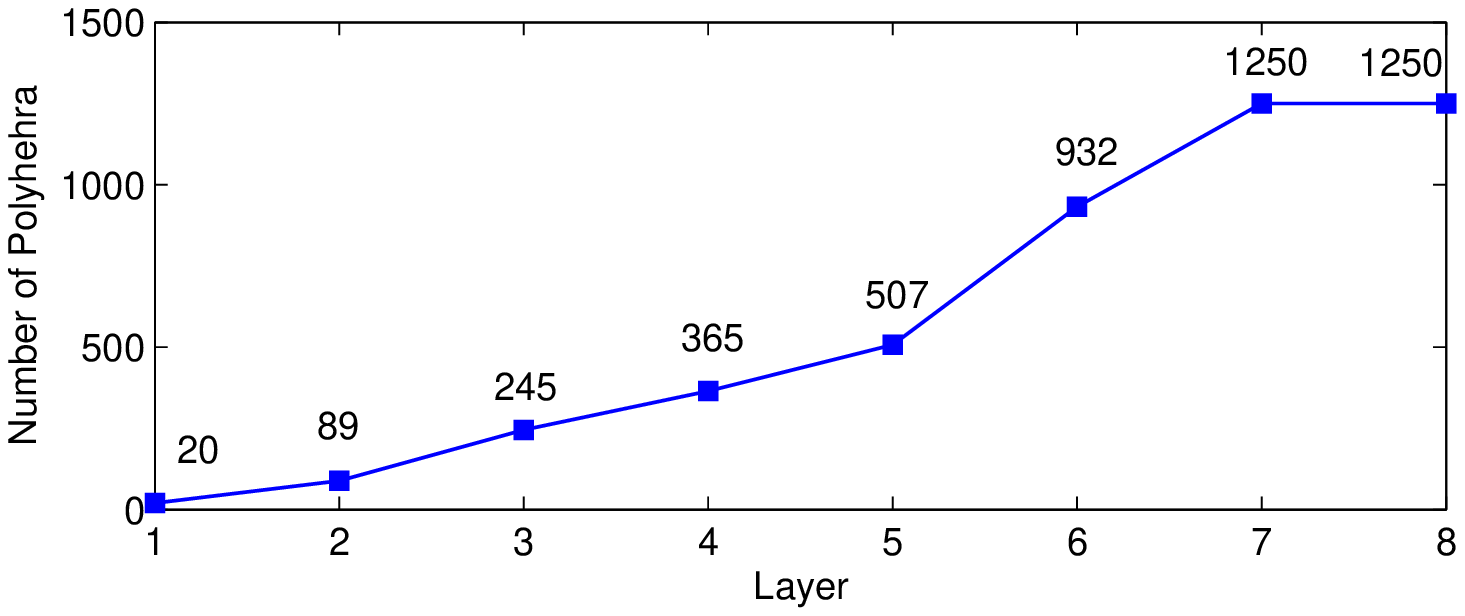}
		\caption{Number of polyhedra produced by each layer is increasing as the layer goes deeper. However, due to the existence of empty sets during the computation process, the number of polyhedra produced in each layer is actually much smaller than the maximal number as shown in Table \ref{tab1}. }
		\label{fig:number}
	\end{center}
\end{figure}

When run Algorithm \ref{algorithm_2}, the number of polyhedra grows from $20$ to $1250$ as the layer goes from $1$ to $8$, which is shown in Figure \ref{fig:number}. It is noted that the output layer is a linear layer so that the number of polyhedra remains unchanged for the output layer 8. Then, in order to reduce the computational time, Algorithm \ref{algorithm_2} is carried out in the framework of parallel computing with a 4 core CPU, and the computational time is reduced from  $1237.915301$ seconds to $704.881292$ seconds, which is shown in Table \ref{tab2}.  This observation can demonstrate that the parallel computing is able to enhance the scalability of the proposed algorithms.

\begin{figure}
	\begin{center}
		\includegraphics[width=12cm]{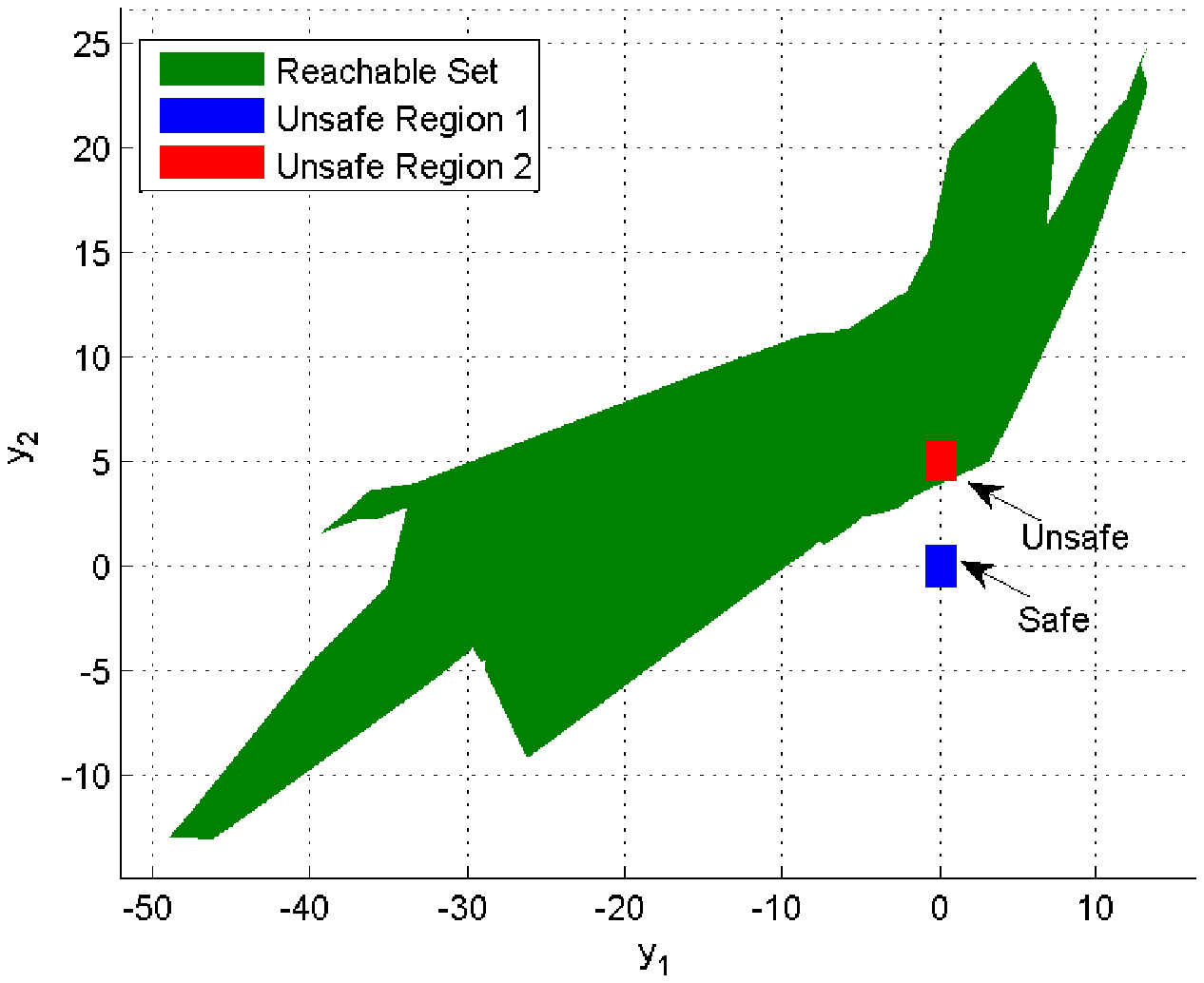}
		\caption{Safety verification regarding unsafe regions. The green area is the computed reachable set, and we consider two unsafe regions depicted in blue and red. (1) The neural network is safe for the blue Unsafe Region 1 since there is no intersection for reachable set and unsafe region; (2) For the red Unsafe Region 2, the neural network is unsafe due to the existence of nonempty intersection. }
		\label{fig:unsafe}
	\end{center}
\end{figure}

\begin{table}[h!]
	\centering
	\caption{Computational Time Comparison with Parallel Computing}\label{tab2}
	\begin{tabular}{|c|c|}
		\hline
		No Parallel Computing & Parallel Computing with 4 Core CPU \\
		\hline
		$1237.915301$ seconds & $704.881292$ seconds \\
		\hline
	\end{tabular}
\end{table}

With the output reachable set depicted in Figure \ref{fig:reachset}, the safety property can be easily checked via function \texttt{VeriReLUNetwork} presented in Algorithm \ref{algorithm_3}. Consider the unsafe region for the output $\mathbf{y}^{[L]}$ is described by  \
\begin{equation*}
\mathcal{S}_{\mathrm{unsafe} }= \{\mathbf{y}^{[L]} \mid \left\|\mathbf{y}^{[L]}-\mathbf{y}^{[L]}_0\right\|_{\infty} \le 1\}
\end{equation*} which can be expressed  in the hyperplane representation of 
\begin{align*}
\mathcal{S}_{\mathrm{unsafe}}=\{\mathbf{y}^{[L]} \mid \mathbf{C} \mathbf{y}^{[L]} \le \mathbf{d}\,~\mathbf{y} \in \mathbb{R}^{{n}^{[L]}}\}	\end{align*}
with $\mathbf{C}$ and $\mathbf{d}$ defined by 
\begin{align*}
\mathbf{C} = \left[ {\begin{array}{*{20}c}
	{1} & {0}  \\
	{-1} & {0}  \\
	{0} & {1}  \\
	{0} & {-1} \\
	\end{array} } \right],~
\mathbf{d} = \left[ {\begin{array}{*{20}c}
	{ 1 + \mathbf{y}_0^{[L]}(1)}   \\
	{1 - \mathbf{y}_0^{[L]}(2)}   \\
	{ 1 + \mathbf{y}_0^{[L]}(3)}   \\
	{1 - \mathbf{y}_0^{[L]}(4)}   \\
	\end{array} } \right]
\end{align*}

We can check the feasibility of linear programming problem formulated by the hyperplane representations of  $\mathcal{Y}^{[L]}$ and $\mathcal{S}_{\mathrm{unsafe}}$ to determine if there exists nonempty intersection. 
If we choose $\mathbf{y}_0 = [0,0]^{\top}$, the ReLU neural  network is safe since there is no intersection between $\mathcal{S}_{\mathrm{unsafe} }$ and $\mathcal{Y}^{[L]}$. However, if we select $\mathbf{y}_0 = [0,5]^{\top}$, the ReLU neural  network is unsafe since $\mathcal{S}_{\mathrm{unsafe} } \cap \mathcal{Y}^{[L]} \ne \emptyset $. Figure \ref{fig:unsafe} shows the safety verification results regarding the two unsafe regions.

\section{Conclusions and Future Remark}
This paper presents a novel output reachable set computation method for neural networks equipped with ReLU activations. With the aid of reachability analysis for ReLU functions, the output reachable set computation for ReLU layer can be made by a set of manipulations for polyhedra, and the output reachable set can be described by a union of a number of polyhedra if the input of the layer is given as a union of polyhedra. Since the involved computation steps are all in the form of operation for polyhedra, they can be efficiently implemented with the help of existing tools for computational geometry. Lastly, on the basis of computed output reachable set, the safety verification for ReLU neural networks can be  done by checking for the nonempty intersection of output reachable set and unsafe  region. 

With the reachable set computed in this paper, it absolutely has potentials to be further applied into many other problems involving ReLU neural networks such as safety-oriented neural network control synthesis problems, which will be our future study based on this paper. Beside the application of the results in this paper, another meaningful research direction is to further improve the scalability of the approach developed in this paper to deal with deeper neural networks. The key step will be finding out computation-efficient method to use as few polyhedra as possible to describe the output sets produced by hidden layers during the computation procedure. This will also be our future study for the reachable set computation problem of neural networks.

\bibliographystyle{ieeetr}
\bibliography{ref}

\end{document}